\documentclass{article}

\usepackage{arxiv}

\usepackage[utf8]{inputenc} 
\usepackage[T1]{fontenc}    
\usepackage{hyperref}       
\usepackage{url}            
\usepackage{booktabs}       
\usepackage{amsfonts}       
\usepackage{nicefrac}       
\usepackage{microtype}      
\usepackage{multirow}

\usepackage{algorithmic}

\usepackage{times}
\usepackage{helvet}
\usepackage{courier}
\usepackage{amsmath}
\usepackage{tikz}
\usepackage{array}
\usepackage{makecell}
\usepackage{amsmath,amsfonts,amsthm,bm}

\usepackage{hyperref}

\newif\ifieee
\ieeefalse

\usepackage{graphicx} 

\usepackage[caption=false,font=footnotesize]{subfig}

\usepackage{letltxmacro}
\LetLtxMacro{\originaleqref}{\eqref}
\renewcommand{\eqref}{eqn.~\originaleqref}
\newtheorem{theorem}{Theorem}

\title{Vanishing Nodes: Another Phenomenon That Makes Training Deep Neural Networks Difficult}

\author{%
  Wen-Yu Chang, \;Tsung-Nan Lin
    \thanks{Wen-Yu Chang and Tsung-Nan Lin were with the Graduate Institute of Communication Engineering, National Taiwan University, Taipei, Taiwan.}
}

\begin{document}

\markboth{Journal of the IEEE Transactions on Neural Networks and Learning Systems}%
{Chang \MakeLowercase{\textit{et al.}}: Vanishing Nodes: Another Phenomena That Makes Training Deep Neural Networks Difficult}

\maketitle

\begin{abstract}
It is well known that the problem of vanishing/exploding gradients is a challenge when training deep networks.
In this paper, we describe another phenomenon, called \textit{vanishing nodes}, that also increases the difficulty of training deep neural networks.
As the depth of a neural network increases, the network's hidden nodes have more highly correlated behavior.
This results in great similarities between these nodes.
The redundancy of hidden nodes thus increases as the network becomes deeper.
We call this problem vanishing nodes,
and we propose
the metric vanishing node indicator (VNI) for quantitatively measuring the degree of vanishing nodes.
The VNI can be characterized by the network parameters, which is shown analytically to be proportional to the depth of the network and inversely proportional to the network width.
The theoretical results show that the effective number of nodes vanishes to one when the VNI increases to one (its maximal value),
and that vanishing/exploding gradients and vanishing nodes are two different challenges that increase the difficulty of training deep neural networks.
The numerical results from the experiments suggest that the degree of vanishing nodes will become more evident during back-propagation training,
and that when the VNI is equal to 1, the network cannot learn simple tasks (e.g. the XOR problem) even when the gradients are neither vanishing nor exploding.
We refer to this kind of gradients as the \textit{walking dead gradients}, which cannot help the network converge when having a relatively large enough scale.
Finally, the experiments show that the likelihood of failed training increases as the depth of the network increases.
The training will become much more difficult due to the lack of network representation capability.
\end{abstract}

\ifieee

\IEEEpeerreviewmaketitle
\fi

\section{Introduction} \label{introduction}
\ifieee
\IEEEPARstart{D}{eep}
\else
Deep
\fi
neural networks (DNN) have succeeded in various fields, including computer vision \cite{alexnet}, speech recognition \cite{speech}, machine translation \cite{google_trans}, medical analysis \cite{medical} and human games \cite{alphago}. Some results are comparable to or even better than those of human experts.

State-of-the-art methods in many tasks have recently used increasingly \textit{deep} neural network architectures. The performance has improved as networks have been made deeper. For example, some of the best-performing models \cite{resnet1, resnet2} in computer vision have included hundreds of layers.

Moreover, recent studies have found that as the depth of a neural network increases, problems, such as vanishing or exploding gradients, make the training process more challenging. \cite{xavier, he} investigated this problem in depth and suggested that initializing the weights in the appropriate scales can prevent gradients from vanishing or exploding exponentially.
\cite{mft:expo, mft:info} also studied (via mean field theory) how vanishing/exploding gradients arise and provided a solid theoretical discriminant to determine whether the propagation of gradients is vanishing/exploding. 

Inspired by previous studies, we investigated the correlation between hidden nodes and discovered that a phenomenon that we call \textit{vanishing nodes} can also affect the capability of a neural network.
In general, the hidden nodes of a neural network become highly correlated as the network becomes deeper.
A correlation between nodes implies their similarity, and a high degree of similarity between nodes produces redundancy.
Because a sufficiently large number of effective nodes is needed to approximate an arbitrary function, the redundancy of the nodes in the hidden layers may debilitate the representation capability of the entire network.
Thus, as the depth of the network increases, the redundancy of the hidden nodes may increase and hence affect the network's trainability. We refer to this phenomenon as that of \textit{vanishing nodes}.

In practical scenarios, redundancy of the hidden nodes in a neural network is inevitable.
For example, a well-trained feed-forward neural network with 500 nodes in each layer to do the 10-class
classification should have many highly correlated hidden nodes. Since we only have 10 nodes
in the output layer, to achieve the classification task, the high redundancy of the 500 hidden nodes is
not surprising, and is even necessary.

We propose a \textit{vanishing node indicator (VNI)}, which is the weighted average of the squared correlation coefficients, as a quantitative metric for the occurrence of vanishing nodes. VNI can be theoretically approximated via the results on the spectral density of the end-to-end Jacobian.
The approximation of VNI depends on the network parameters, including the width, depth, distribution of weights, and the activation functions, and is shown to be simply proportional to the depth of the network and inversely proportional to the width of the network.

When the VNI of a network is equal to 1 (i.e. all the hidden nodes are
highly correlated, the redundancy of nodes is at a maximal level),
we call this situation \textit{network collapse}.
The representation power is not sufficient, and hence the network cannot successfully learn most training tasks.
The network collapse theorem and its proof are provided to show that the effective number of nodes vanishes to 1 as the VNI approaches 1.
In addition, the numerical results show that back-propagation training also intensifies the correlations of the hidden nodes when we consider a deep network.

Also, another weight initialization method is proposed for tweaking the initial VNI of a network to 1 even when the network has a relatively small depth.
This weight initialization meets the "norm-preserving" condition of \cite{xavier, mft:spectral}.
Numerical results for this weight initialization method show that when the VNI is set to 1, the network cannot learn simple tasks (e.g. the XOR problem) even when vanishing/exploding gradients do not appear.
This implies that when the VNI is close to 1, the back-propagated gradients, which are neither vanishing nor exploding, cannot successfully train the network, and hence we call this kind of gradient \textit{walking dead gradients}.

Finally, we show that vanishing/exploding gradients and vanishing nodes are two different problems because the two problems may arise from different specific conditions.
The experimental results show that the likelihood of failed training increases as the depth of the network increases. The training will become much more difficult due to the lack of network representation capability. 

This paper is organized as follows: some related works are discussed in Section \ref{related}. The vanishing nodes phenomenon is introduced in Section \ref{why}. A theoretical analysis and quantitative metric are presented in Section \ref{why}. Section \ref{compare} compares the problem of vanishing nodes with the problem of vanishing/exploding gradients. Section \ref{experiments} presents the experimental results and Section \ref{conclusion} presents our conclusions.

\section{Related Work} \label{related}

Problems in the training of deep neural networks have been encountered in several studies.
For example, \cite{xavier, he} investigated vanishing/exploding gradient propagation and gave weight initialization methods as the solution. \cite{evop} suggested that vanishing/exploding gradients might be related to the sum of the reciprocals of the widths of the hidden layers.
\cite{opt_prob, saddle} stated that for training deep neural networks, saddle points are more likely to be a problem than local minima.
\cite{degrade1, degrade2, resnet1} explained the \textit{degradation} problem: the performance of a deep neural network degrades as its depth increases.

The correlation between the nodes of hidden layers within a deep neural network is the main focus of the present paper, and several kinds of correlations have been discussed in the literature.
\cite{mft:info} surveyed the propagation of the correlation between two different inputs after several layers.
\cite{whiten1, whiten2} suggested that the input features must be whitened (i.e., zero-mean, unit variances and uncorrelated) to achieve a faster training speed.

Dynamical isometry is one of the conditions that make ultra-deep network training more feasible.
\cite{mft:linear} reported the dynamical isometry theoretically ensures a depth-independent learning speed.
\cite{mft:sigmoid, mft:spectral} suggested several ways to achieve dynamical isometry for various settings of network architecture, and \cite{mft:cnn, mft:rnn} practically trained ultra-deep networks in various tasks.

\section{Vanishing Nodes: Correlation Between Hidden Nodes} \label{why}

In this section, the correlations between neurons in the hidden layers are investigated.
If two neurons are highly correlated (for example, the correlation coefficient is equal to $+1$ or $-1$), one of the neurons becomes redundant.
A great similarity between nodes may reduce the effective number of neurons within a network.
In some cases, the correlation of hidden nodes may disable the entire network. This phenomenon is called \textit{vanishing nodes}.


First, consider a deep feed-forward neural network with depth $L$.
For simplicity of analysis, we assume all layers have the same width $N$.
The weight matrix of layer $l$ is $\mathbf{W}_l\in \mathbb{R}^{N\times N}$, the bias of layer $l$ is $\mathbf{b}_l\in \mathbb{R}^N$ (a column vector), and the common activation function of all layers is $\phi(\cdot):\mathbb{R}\rightarrow \mathbb{R}$. The input of the network is $\mathbf{x}_0$, and the nodes at output layer $L$ are denoted by $\mathbf{x}_L$. The pre-activation of layer $l$ is $\mathbf{h}_l\in \mathbb{R}^N$ (a column vector), and the post-activation of layer $l$ is $\mathbf{x}_l\in \mathbb{R}^N$ (a column vector). That is, $\forall l \in \{1, ..., L\}$,

\begin{equation}
    \mathbf{h}_l=\mathbf{W}_l\mathbf{x}_{l-1}+\mathbf{b}_l,\;\;\;
    \mathbf{x}_{l}=\phi(\mathbf{h}_l).
\label{network_eqn}
\end{equation}

The variance of node $i$ is defined as $\sigma_i^2\overset{\Delta}{=}\mathbb{E}_{\mathbf{x}_0}[(x_{l(i)}-\overline{x_{l(i)}})^2]$, and the squared correlation coefficient ($\rho_{ij}^2$) between nodes $i$ and $j$ can be computed as $\rho_{ij}^2\overset{\Delta}{=}
\frac
{\mathbb{E}_{\mathbf{x}_0}[(x_{l(i)}-\overline{x_{l(i)}})(x_{l(j)}-\overline{x_{l(j)}})]^2}
{\mathbb{E}_{\mathbf{x}_0}[(x_{l(i)}-\overline{x_{l(i)}})^2]\mathbb{E}_{\mathbf{x}_0}[(x_{l(j)}-\overline{x_{l(j)}})^2]},$
where $\rho_{ij}^2$ ranges from $0$ to $1$.
Nodes $x_{l(i)}$ and $x_{l(j)}$ are highly correlated only if the magnitude of the correlation coefficient between them is nearly 1. $\rho_{ij}^2$ indicates the degree of similarity between node $i$ and node $j$.
If $\rho_{ij}$ is close to $+1$ or $-1$, then node $i$ can be approximated in a linear fashion by node $j$.
A great similarity indicates redundancy.
If the nodes of the hidden layers exhibit great similarity, the effective number of nodes will be much lower than the original network width. Therefore, we call this phenomenon the \textit{vanishing nodes problem}.




In the following section, we propose a metric to quantitatively describe the degree of vanishing of the nodes of a deep feed-forward neural network.
A theoretical analysis of this metric indicates that it is proportional to the depth of the network and inversely proportional to the width of the network.
This quantity is shown analytically to depend on the statistical properties of the weights and the nonlinear activation function. 


\subsection{Vanishing Node Indicator} \label{initial}

Consider the network architecture defined in \eqref{network_eqn}. In addition, the following assumptions are made: (1) The input $\mathbf{x}_0$ is zero-mean, and the features in $\mathbf{x}_0$ are independent and identically distributed. (2) All weight matrices $\mathbf{W}_l$ in each layer are initialized from the same distribution with variance $\sigma_w^2/N$. (3) All the bias vectors $\mathbf{b}_l$ in each layer are initialized to zero.


The input--output Jacobian matrix $\mathbf{J}\in\mathbb{R}^{N\times N}$  is defined as the first-order partial derivative of the output layer with respect to the input layer, which can be rewritten as $\frac{\partial\mathbf{x}_L}{\partial\mathbf{x}_0}=\prod_{l=1}^{L}\mathbf{D}_l\mathbf{W}_l$,
where $\mathbf{D}_l\overset{\Delta}{=} diag(\phi'(\mathbf{h}_l))$ is the derivative of the point-wise activation function $\phi$ at layer $l$.
To conduct an analysis similar to that of \cite{mft:linear}, consider the first-order forward approximation:
$\mathbf{x}_L-\overline{\mathbf{x}_L} \approx \mathbf{Jx}_0$. Therefore, the covariance matrix of the nodes ($\mathbf{C}\in\mathbb{R}^{N\times N}$) at the output layer  can be computed as

\begin{equation}
    \begin{aligned}
    \mathbf{C} 
    &\overset{\Delta}{=}
    \mathbb{E}_{\mathbf{x}_0}[(\mathbf{x}_L-\overline{\mathbf{x}_L})(\mathbf{x}_L-\overline{\mathbf{x}_L})^T]
    \approx
    \mathbb{E}_{\mathbf{x}_0}[(\mathbf{Jx}_0)(\mathbf{Jx}_0)^T]\\
    &=
    \mathbf{J}\mathbb{E}_{\mathbf{x}_0}[\mathbf{x}_0\mathbf{x}_0^T]\mathbf{J}^T
    =
    \sigma_x^2\mathbf{J}\mathbf{J}^T,
    \end{aligned}
    \label{covariance_eqn}
\end{equation}

where $\sigma_x^2$ is the common variance of the features in $\mathbf{x}_0$, and the expected values are calculated with respect to the input $\mathbf{x}_0$. For notational simplicity, we omit the subscript $\mathbf{x}_0$ of the expectations in the following equations. 
It can be easily derived that the squared covariance of nodes $i$ and $j$ is equal to the product of the squared correlation coefficient and the two variances. That is, $[C_{(ij)}]^2=\rho_{ij}^2\sigma_i^2\sigma_j^2$.

In this paper, we propose the \textit{vanishing node indicator} (VNI) $R_{sq}$ to quantitatively characterize the degree of vanishing nodes for a given network architecture. It is defined as follows:

\begin{equation}
    R_{sq}\overset{\Delta}{=}
    \frac{\sum_{i=1}^N\sum_{j=1}^N\rho_{ij}^2\sigma_i^2\sigma_j^2}
{\sum_{i=1}^N\sum_{j=1}^N\sigma_i^2\sigma_j^2}.
\label{rsq_def}
\end{equation}

VNI calculates the weighted average of the squared correlation coefficients $\rho_{ij}^2$ between output layer nodes with non-negative weights $\sigma_i^2\sigma_j^2$. Basically, VNI $R_{sq}$, which ranges from $1/N$ to $1$, summarizes the similarity of the nodes at the output layer.
If all nodes are independent of each other, the correlation coefficients $\rho_{ij}$ will be 0 (if $i\neq j$) or 1 (if $i=j$) and $R_{sq}$ will take on its minimum, which is $1/N$.
Otherwise, if all of the output nodes are highly correlated, then all the squared correlation coefficients $\rho_{ij}^2$ will be nearly 1, and therefore $R_{sq}$ will reach its maximum value $1$.
Note that the weights $\sigma_i^2\sigma_j^2$ in the weighted average can be interpreted as the importance of the output-layer nodes $i$ and $j$. If all of the output layer nodes have equal variances, VNI $R_{sq}$ is simply the average of the squared correlation coefficients $\rho_{ij}^2$.

\begin{figure*}[ht]
\centering
\subfloat[Network width $N=200$]{
  \centering
  \includegraphics[width=0.49\linewidth,trim={0 0 0 0.8cm},clip]{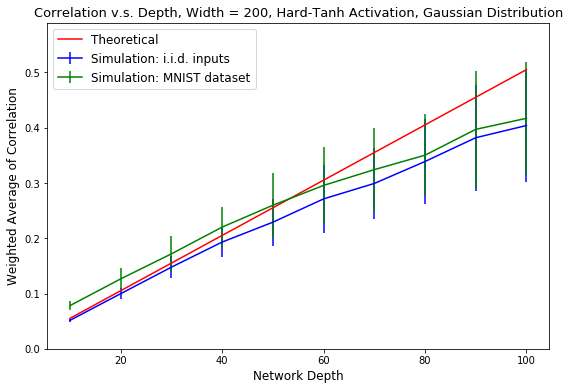}
  \label{fig:sec4_sim2_a}
}%
\subfloat[Network width $N=500$]{
  \centering
  \includegraphics[width=0.49\linewidth,trim={0 0 0 0.8cm},clip]{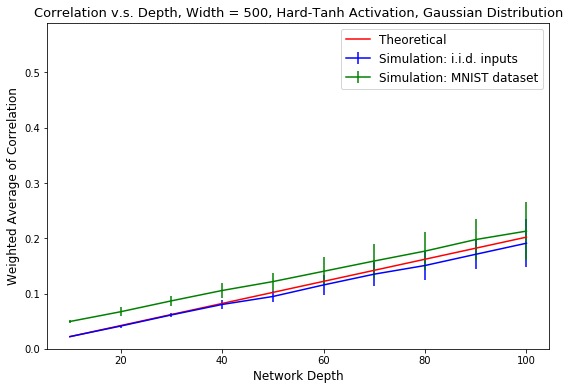}
  \label{fig:sec4_sim2_b}
}

\caption{
The results of VNI $R_{sq}$ with respect to network depth $L$ for network widths 200 and 500. The red line is calculated from \eqref{rsq_moment}, the blue line is computed from \eqref{rsq_def} with the input data of zero mean and i.i.d input data, and the green line is computed from \eqref{rsq_def} with MNIST data.
The VNI $R_{sq}$ expressed in \eqref{rsq_moment} is very close to the original definition in \eqref{rsq_def}.
}
\label{fig:sec4_sim2}
\end{figure*}

With the covariance matrix defined in \eqref{covariance_eqn} and the formulas for the traces of matrices, VNI $R_{sq}$ can be expressed in terms of the covariance matrix as
\begin{equation}
    \begin{aligned}
    R_{sq}
    &=\frac{
    \sum_{i=1}^N\sum_{j=1}^N\mathbb{E}_{\mathbf{x}_0}
    [(x_{L(i)}-\overline{x_{L(i)}})(x_{L(j)}-\overline{x_{L(j)}})]^2
    }{
    \sum_{i=1}^N\sum_{j=1}^N
    \mathbb{E}_{\mathbf{x}_0}[(x_{L(i)}-\overline{x_{L(i)}})^2]
    \mathbb{E}[(x_{L(j)}-\overline{x_{L(j)}})^2]
    }\\
    &=
    \frac{\sum_{i=1}^N\sum_{j=1}^N[C_{(ij)}]^2}
    {\sum_{i=1}^N\sum_{j=1}^NC_{(ii)}C_{(jj)}}
    =
    \frac{tr(\mathbf{C}{\mathbf{C}}^T)}
    {tr(\mathbf{C})^2}
    ,
    \end{aligned}
    \label{rsq_eqn}
\end{equation}
where $tr(\cdot)$ denotes the trace of a matrix.

From \eqref{covariance_eqn}, substituting $\sigma_x^2\mathbf{JJ}^T$ for $\mathbf{C}$ in \eqref{rsq_eqn}, and noting that $tr(\mathbf{A}^k)$ is equal to the sum of the $k$th powers of the eigenvalues of the symmetric matrix $\mathbf{A}$ \cite{matrix}, an approximation of $R_{sq}$ can be obtained:

\begin{equation}
    \begin{aligned}
    R_{sq}
    &\approx
    \frac{tr(\mathbf{JJ}^T\mathbf{JJ}^T)}{tr(\mathbf{JJ}^T)^2}
    =\frac{\sum_{k=1}^N\lambda_k^2}{(\sum_{k=1}^N\lambda_k)^2}\\
    &=\frac{N\cdot m_2}{(N\cdot m_1)^2}
    =\frac{m_2}{Nm_1^2},
    \end{aligned}
    \label{rsq_eigen}
\end{equation}

where $\lambda_k$ is the $k$th eigenvalue of $\mathbf{JJ}^T$, and $m_i$ is the $i$th moment of the eigenvalues of $\mathbf{JJ}^T$.



In \eqref{rsq_eigen}, we show that $R_{sq}$ is related to the expected moments of the eigenvalues of $\mathbf{JJ}^T$. Because the moments of the eigenvalues of $\mathbf{JJ}^T$ have been analyzed in previous studies \cite{mft:spectral}, we can leverage the recent results  by \cite{mft:spectral}: $m_1=(\sigma_w^2\mu_1)^L$, and $m_2=(\sigma_w^2\mu_1)^{2L}L\big(\frac{\mu_2}{\mu_1^2}+\frac{1}{L}-1-s_1\big)$,
where $\sigma_w^2/N$ is the variance of the initial weight matrices, $s_1$ is the first moment of the series expansion of the S-transform associated with the weight matrices, and $\mu_k$ is the $k$th moment of the series expansion of the moment generating function associated with the activation functions.
If we insert the expressions of $m_1$ and $m_2$ into \eqref{rsq_eigen}, we can obtain an approximation of the expected VNI:

\begin{equation}
    R_{sq}\approx \frac{L}{N}\Big(\frac{\mu_2}{\mu_1^2}+\frac{1}{L}-1-s_1\Big)
    =
    \frac{1}{N}+\frac{L}{N}\Big(\frac{\mu_2}{\mu_1^2}-1-s_1\Big)
    ,
    \label{rsq_moment}
\end{equation}

which shows that VNI is determined by the depth $L$, the width $N$, the moments of the activation functions $\mu_k$ and the statistical properties of the weights $s_1$. 
Because $R_{sq}$ ranges from $1/N$ to $1$, the approximation in \eqref{rsq_moment} is more accurate when $N>>L$.
Moreover, it can be easily seen that the correlation is inversely proportional to the width $N$ of the network, and proportional to the depth $L$ of the network.

To evaluate the accuracy of \eqref{rsq_moment} with respect to the original definition in \eqref{rsq_def}, we have designed the following experiments.
A network width, $N\in\{200, 500\}$, is set. The network depth $L$ is adjusted from $10$ to $100$ with the Hard-Tanh activation function.
One thousand data points with the distribution $\mathbf{x}_0\sim Gaussian(\mu_x=0, \sigma^2_x=0.1)$ and 50,000 training images from the MNIST dataset \cite{mnist} are fed into the network.
In each network architecture, the weights are initialized with scaled-Gaussian distribution \cite{xavier} of various random seeds for 100 runs.
The $R_{sq}$ calculated from \eqref{rsq_def} is then recorded to compute the mean and the standard deviation with respect to various network depths $L$.
The results are shown in Fig. \ref{fig:sec4_sim2} as the blue and green lines denoted ``Simulation i.i.d. inputs'' and ``Simulation MNIST dataset.''
The red line denoted by ``Theoretical'' is the result calculated from \eqref{rsq_moment}.
This experiment demonstrates that the VNI expressed in terms of the network parameters in \eqref{rsq_moment} is very close to the original definition in \eqref{rsq_def}.
Similar results are obtained with different activation functions (e.g., Linear, ReLU) and different weight initializations (e.g., scaled uniform distribution).

Fig. \ref{fig:sec4_sim3} plots the squared correlation coefficients between output nodes, which are evaluated with 50,000 training images in the MNIST dataset \cite{mnist} for various network architectures. White indicates no correlation, and black means that $\rho_{ij}^2 = 1$. Fig. \ref{fig:sec4_sim3} (a) plots the squared correlation coefficients for four architectures with the same network width ($N=200$) at different depths (5, 50, 300, and 1000). Fig. \ref{fig:sec4_sim3} (b) shows the architectures with the same depth ($L=100$) and different widths (5, 50, 200, 1000).  
This shows that the vanishing node phenomenon becomes evident with increasing depth and is inversely proportional to the width.

\subsection{Network collapse} \label{collapsing}

When the VNI of a network is equal to 1 (i.e. all the hidden nodes are
highly correlated, the redundancy of the nodes is at a maximal level),
and so the network collapses to only one effective node.
The representation power is not sufficient, and hence the network cannot successfully
learn most training tasks, e.g., classification tasks with more than two classes, or non-linearly separable classification tasks.

As theoretical evidence for this claim, we present the \textit{network collapse theorem}:

\begin{theorem}[Network collapsing]
    \label{collapse}
    The effective number of nodes becomes 1 when the VNI $R_{sq}$ is 1.
\end{theorem}

\begin{proof}

    First, the $N$ random variables of the values of the nodes in a hidden layer are defined as $\{X_1, X_2, ..., X_N\}$.
    Without loss of generality, we assume that each $X_i$ follows a $\mathcal{N}(0, 1)$ distribution.
    Therefore, the covariance matrix of the random vector $[X_1, X_2, ..., X_N]^T$ is
    \begin{equation}
        \mathbf{C}=
        \begin{bmatrix}
            1 & \rho_{12} & \cdots & \rho_{1N} \\
            \rho_{21} & 1 & \cdots & \rho_{2N} \\
            \vdots & \vdots & \ddots & \vdots  \\
            \rho_{N1} & \rho_{N2} & \cdots & 1
        \end{bmatrix},
        \label{rv_cov}
    \end{equation}
    where $C_{ij}=\mathbb{E}[(X_i-\overline{X_i})(X_j-\overline{X_j})]$ as defined in
    \eqref{covariance_eqn}, $\rho_{ij}$ is the correlation coefficient between $X_i$ and $X_j$,
    and hence $\mathbf{C}$ is a symmetric matrix.
    By the definition of the VNI $R_{sq}$ from \eqref{rsq_def}, the $R_{sq}$
    can be represented as
    \begin{equation}
        R_{sq}=\frac{\sum_{i=1}^N\sum_{j=1}^N\rho_{ij}^2}{N^2}
        =\frac{tr(\mathbf{C}\mathbf{C}^T)}{N^2}=\frac{1}{N^2}tr(\mathbf{C}^2).
        \label{rv_rsq}
    \end{equation}

    Let the eigenvalues of $\mathbf{C}$ be $\lambda_1, \lambda_2, \dots, \lambda_N$.
    By the relation between trace of a matrix and its eigenvalues, we have
    \begin{equation}
        \begin{aligned}
            \sum_{i=1}^N\lambda_i&=tr(\mathbf{C})=N\\
        \sum_{i=1}^N\lambda_i^2&=tr(\mathbf{C}^2)=N^2R_{sq}.
        \end{aligned}
        \label{rv_eigen}
    \end{equation}

    To relate $R_{sq}$ with the redundancy of the random variables, a method for measuring the redundancy
    is needed.
    From a principle component analysis (PCA), the eigenvalue of $\mathbf{C}$ can represent the 
    energy (i.e. the variance) associated with each eigenvector.
    Therefore, we can use the distribution of the eigenvalues $\lambda_i$ to determine the proportion of
    redundant components, and hence the effective number of nodes can be evaluated.
    Similar to PCA, we first rearrange the order of the eigenvalues $\lambda_i$ so that
    $\lambda_1\geq\lambda_2\geq\cdots\geq\lambda_N\geq 0$.
    We suppose given a constant, $\varepsilon\in(0, 1)$, the \textit{effective threshold ratio} of the eigenvalues.
    That is, if $\lambda_i \geq \varepsilon\lambda_1$, then we say that the $i$th component $\lambda_i$ is
    $\varepsilon$-effective.
    Otherwise, the $i$th component $\lambda_i$ is said to be redundant.
    
    We introduce a new metric called the $\varepsilon$-effective number of nodes ($\varepsilon$-ENN):
    \begin{equation}
        \varepsilon\text{-ENN}
        \equiv N_{e}^\varepsilon
        \overset{\Delta}{=}max(\{t\in\mathbb{N}: \lambda_t\geq\varepsilon\lambda_1\}).
        \label{eENN_def}
    \end{equation}
    That is, $\varepsilon\text{-ENN}$ is the maximum number of $\varepsilon$-effective nodes.
    As in \eqref{rv_eigen}, the constraints on $\lambda_i$ are already derived.
    Also, it is intuitive that the maximum in \eqref{eENN_def} can simply be attained with eigenvalues
    $\{\lambda_1, \varepsilon\lambda_1, \dots, \varepsilon\lambda_1, 0, \dots, 0\dots, 0\}$, where there are 
    $(N_{e}^\varepsilon-1)$ $\varepsilon\lambda_1$ and $(N-N_e^\varepsilon)$ zeros.
    Inserting these eigenvalues into \eqref{rv_eigen}, we get
    \begin{equation}
        \begin{aligned}
            \lambda_1+(N_e^\varepsilon-1)\varepsilon\lambda_1 &= N \\
            \lambda_1^2+(N_e^\varepsilon-1)(\varepsilon\lambda_1)^2 &= N^2R_{sq}.
        \end{aligned}
        \label{effective_eigen}
    \end{equation}
    Inserting the first equation in \eqref{effective_eigen} into the second one, we get
    \begin{equation}
        [1+(N_e^\varepsilon-1)\varepsilon]^2
        =\Big(\frac{N}{\lambda_1}\Big)
        =\frac{1+(N_e^\varepsilon-1)\varepsilon^2}{R_{sq}},
        \label{quad_eenn}
    \end{equation}
    which is a solvable quadratic equation.

    Making use of the condition $R_{sq}=1$, we have
    \begin{equation}
        \begin{aligned}
            &[1+(N_e^\varepsilon-1)\varepsilon]^2=1+(N_e^\varepsilon-1)\varepsilon^2\\
            \Rightarrow&(N_e^\varepsilon-1)[(N_e^\varepsilon-2)\varepsilon+2]=0\\
            \Rightarrow& N_e^\varepsilon=1
        \end{aligned}
        \label{collapse_eqn}
    \end{equation}

    The solution for $N_e^\varepsilon$ in \eqref{quad_eenn} is $1$
    (no matter the value of $\varepsilon$).
    Therefore, the $\varepsilon$-effective number of nodes $N_e^\varepsilon$ becomes only 1
    when the VNI $R_{sq}$ is 1.
\end{proof}

Therefore, when the VNI approaches 1,
the network degrades to a single perceptron, and hence the
representation power of the network is insufficient to solve practical tasks (e.g., classification tasks with more than two classes or non-linearly separable classification tasks).




\subsection{Impact of back-propagation} \label{backprop}

In Section \ref{initial}, we showed that the correlation of a network will increase as the depth $L$ increases; in this section, we exploit the manner in which the back-propagation training process will influence the network correlation by the following experiments. 


\begin{figure*}
\centering
\newcommand{\myWidth}{0.8\textwidth}

\subfloat[Network width $N=200$]{
  \centering
  \includegraphics[width=0.8\textwidth]{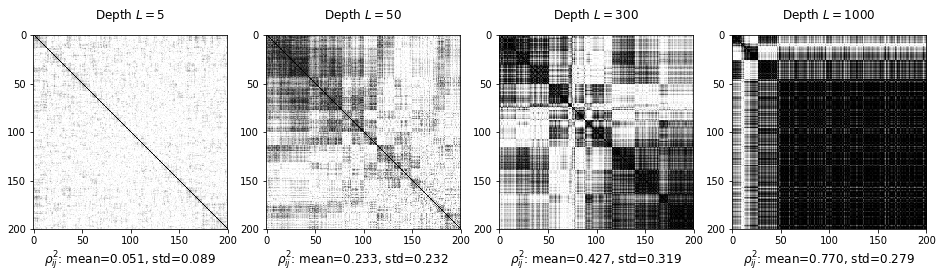}
  \label{fig:sec4_sim3_a}
}\hspace{3mm}%
\subfloat{
  \centering
  \includegraphics[width=10mm]{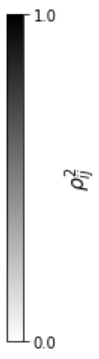}
}%

\subfloat[Network depth $L=100$]{
  \centering
  \includegraphics[width=0.8\textwidth]{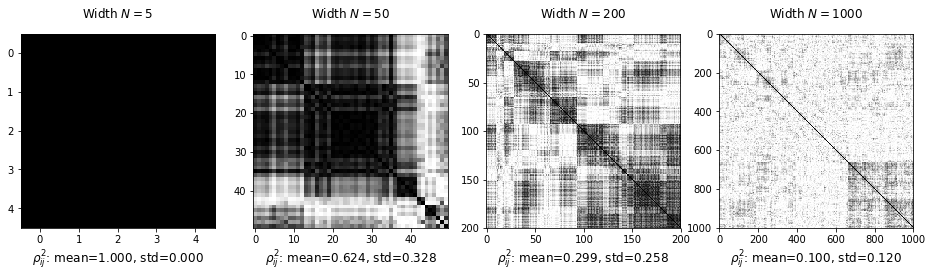}
  \label{fig:sec4_sim3_b}
}\hspace{3mm}%
\subfloat{
  \centering
  \includegraphics[width=10mm]{"colorbar_bw"}
}%

\caption{The magnitudes of correlation coefficient $\rho_{ij}$ between output nodes.
The black color means $\rho_{ij}^2=1$ while the white color indicates $\rho_{ij}^2=0$.
The top row shows that the correlation is positively related to the network depth $L$, and the bottom row that the correlation is negatively related to the network width $N$. Note that we have rearranged the node index to cluster the correlated nodes.}
\label{fig:sec4_sim3}
\end{figure*}

First, the same architecture defined in \eqref{network_eqn}, with $L=100$, $N=500$, tanh activation,  and scaled Gaussian initialization \cite{xavier}, is used. The network is then trained on the MNIST dataset \cite{mnist} and optimized with stochastic gradient descent (SGD) with a batch size of 100. The network is trained with three different learning rates for different seeds to initialize the weights for 20 runs. We then record the quartiles of VNI ($R_{sq}$) with respect to the training epochs, as shown in Fig. \ref{fig:sec5_sim1}.

The boundaries of the colored areas represent the first and the third quartiles (i.e., the 25th and the 75th percentiles), and the line represents the second quartile (i.e., the median) of $R_{sq}$ over 20 trials.
This shows that in some cases, VNI increases to 1 during the training process, otherwise VNI grows larger initially, and then decreases to a value which is larger than the initial VNI.
A severe intensification of VNI may occur, as shown by the blue line, which is trained at the learning rate of $10^{-2} $. 
Moreover, we observe that training will become much more difficult due to a lack of network representation capability as VNI $R_{sq}$ approaches 1.
Further discussion is provided in Section \ref{experiments} to investigate, by the use of various training parameters, the impact of VNI.

\begin{figure*}
\centering

\newcommand{\myWidth}{0.95\linewidth}
\includegraphics[width=\myWidth]{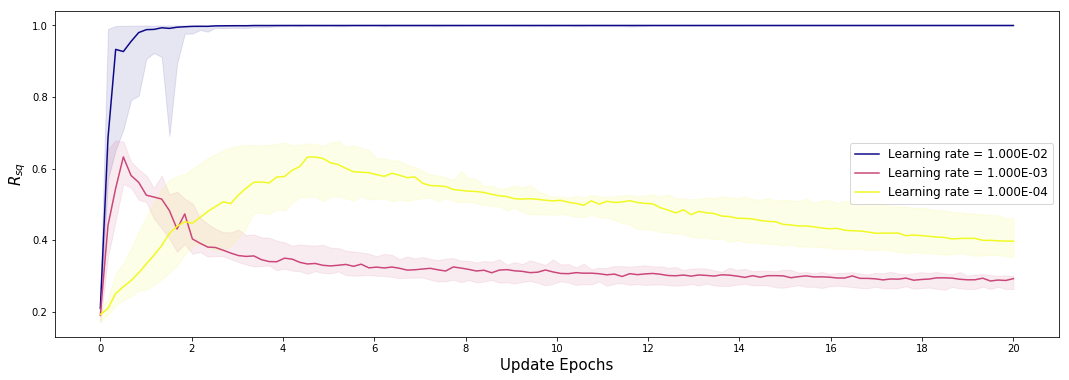}
\caption{Severe intensification of VNI (increases to 1) may occur, as shown by the blue line, which is trained with a learning rate of $10^{-2} $. Otherwise VNI rises initially, and then decreases to a value which is larger than the initial VNI.
}
\label{fig:sec5_sim1}
\end{figure*}

\section{Comparison with Exploding/Vanishing Gradients} \label{compare}

In this section, we explore whether the vanishing nodes phenomenon arises from the problem of exploding/vanishing gradients.
Exploding/vanishing gradients in deep neural networks are a problem regarding the scale of forward-propagated signals and back-propagated gradients, which exponentially explode/vanish as the networks grows deeper.
We perform a theoretical analysis of exploding/vanishing gradients and show analytically the difference between them and the newly identified problem of vanishing nodes.


As in a previous study \cite{xavier}, we use the variances of the hidden nodes to evaluate the scales of the back-propagated gradients. Consider the model and the assumptions in Section \ref{why} and an additional assumption: the gradient $\frac{\partial Cost}{\partial \mathbf{x}_L}$ of  the output layer is a zero-mean i.i.d. random (row) vector.
That is, $\mathbb{E}[\mathbf{x}_0{\mathbf{x}_0}^T] = \sigma_x^2\cdot\mathbf{I}$ and $\mathbb{E}\Big[\Big(\frac{\partial Cost}{\partial \mathbf{x}_L}\Big)^T\frac{\partial Cost}{\partial \mathbf{x}_L}\Big] = \sigma_y^2\cdot\mathbf{I}$,
where $\sigma_x^2$ and $\sigma_y^2$ are defined as the variances of the input layer nodes and output layer gradients, respectively.
Consider the variances $Var[\mathbf{x}_L]$ and $Var\Big[\frac{\partial Cost}{\partial \mathbf{x}_0}\Big]$ of the gradients of the output nodes and the input nodes, respectively.
The phenomena of exploding/vanishing gradients occur only if the scales of the forward and backward propagation exponentially increase or decrease as the depth increases.
This means that the magnitude of the gradients will be bounded if we can prevent the scales of forward and backward propagation from exploding or vanishing. 

According to the assumptions in Section \ref{why} and  \eqref{covariance_eqn}, we can approximate the shared scalar variance of all output nodes $Var[\mathbf{x}_L]\in\mathbb{R}$ and the shared scalar variance of all input gradients $Var\Big[\frac{\partial Cost}{\partial \mathbf{x}_0}\Big]\in\mathbb{R}$ by

\begin{equation}
    \begin{aligned}
    Var[\mathbf{x}_L] &=\mathbb{E}[(\mathbf{x}_L-\overline{\mathbf{x}_L})^T(\mathbf{x}_L-\overline{\mathbf{x}_L})]/N\\
    &\approx \mathbb{E}[(\mathbf{Jx}_0)^T\mathbf{Jx}_0]/N\\
    &=\mathbb{E}[tr(\mathbf{J}^T\mathbf{J}\mathbf{x}_0\mathbf{x}_0^T)]/N
    =\sigma_x^2\cdot tr(\mathbf{J}^T\mathbf{J})/N
    \label{xvar_to_trace}
    \end{aligned}
\end{equation}

\begin{equation}
    \begin{aligned}
    Var\Big[&\frac{\partial Cost}{\partial \mathbf{x}_0}\Big]
    \\
    &=\mathbb{E}\Big[\Big(\frac{\partial Cost}{\partial \mathbf{x}_0}-\overline{\frac{\partial Cost}{\partial \mathbf{x}_0}}\Big)
    \Big(\frac{\partial Cost}{\partial \mathbf{x}_0}-\overline{\frac{\partial Cost}{\partial \mathbf{x}_0}}\Big)^T\Big]\Big/N
    \\
    &= \mathbb{E}\Big[
    \Big(\frac{\partial Cost}{\partial \mathbf{x}_L}\mathbf{J}\Big)
    \Big(\frac{\partial Cost}{\partial \mathbf{x}_L}\mathbf{J}\Big)^T\Big]\Big/N
    \\
    &=\sigma_y^2\cdot tr(\mathbf{J}^T\mathbf{J})/N,
    \end{aligned}
    \label{yvar_to_trace}
\end{equation}

where the chain rule for back-propagation, $\frac{\partial Cost}{\partial \mathbf{x}_0}=\frac{\partial Cost}{\partial \mathbf{x}_L}\frac{\partial \mathbf{x}_L}{\partial \mathbf{x}_0}=\frac{\partial Cost}{\partial \mathbf{x}_L}\mathbf{J}$, is used, and the shared scalar variance of a vector is the average of the variances of all its components.
Also, it is already known that
$tr(\mathbf{J}^T\mathbf{J})=N\cdot m_1=N\cdot(\sigma_w^2\mu_1)^L$. Thus, we have $Var[\mathbf{x}_L]=\sigma_x^2(\sigma_w^2\mu_1)^L$ and $Var\Big[\frac{\partial Cost}{\partial \mathbf{x}_0}\Big]=\sigma_y^2(\sigma_w^2\mu_1)^L$,
where $\sigma_w^2=N\cdot Var[W_{ij}]$, and $\mu_1$ is the first moment of the nonlinear activation function. It is obvious that the variances of both forward and backward propagation will neither explode nor vanish if and only if $(\sigma_w^2\mu_1)=1$.

For the weight gradient of the hidden layer $l$, the variance can be used to measure the scale distribution. Because $\frac{\partial Cost}{\partial \mathbf{W}_l}
=\mathbf{x}_{l-1}\cdot\frac{\partial Cost}{\partial \mathbf{h}_l}$ and both $\mathbf{x}_{l-1}$ and $\frac{\partial Cost}{\partial \mathbf{h}_l}$ are assumed to be zero-mean, the variance of the weight gradient can be evaluated as

\begin{equation}
    Var\Big[\frac{\partial Cost}{\partial \mathbf{W}_l}\Big]
    =
    Var[\mathbf{x}_{l-1}]\cdot
    Var\Big[\frac{\partial Cost}{\partial \mathbf{h}_l}\Big]
    \approx
    \sigma_x^2\sigma_y^2(\sigma_w^2\mu_1)^{L-1},
    \label{weight_var}
\end{equation}
where we can evaluate $Var[\mathbf{x}_{l-1}]$ and $Var\Big[\frac{\partial Cost}{\partial \mathbf{h}_l}\Big]$ using the results of the forward/backward variance propagation and split the entire network into two sub-networks. One sub-network has the input layer $\mathbf{x}_{0}$ and output layer $\mathbf{x}_{l-1}$, and the other sub-network has the input layer $\mathbf{x}_{l}$ and the output layer $\mathbf{x}_{L}$. Note that \eqref{weight_var} also implies that if, and only if, $(\sigma_w^2\mu_1)=1$, the weight gradients will never explode or vanish.

However, \eqref{rsq_moment} shows that VNI ($R_{sq}$) may still increase with the depth of the network even if $(\sigma_w^2\mu_1)=1$.
That is, the occurrence of the vanishing nodes becomes evident when $(\mu_2/\mu_1^2-1-s_1)$ is close to 1, whereas vanishing/exploding gradients occur when $(\sigma_w^2\mu_1)$ is far from 1. If the network's initialization parameter is appropriately set, so that $(\sigma_w^2\mu_1)$ is close to 1, $R_{sq}$ may still tend to 1 due to the network depth, the activation function, and the weight distribution.
Therefore, from \eqref{rsq_moment} and \eqref{weight_var}, it is clear that the problem of vanishing nodes may occur regardless of whether there are exploding/vanishing gradients.

\section{Invoke the Training Difficulty of Vanishing Nodes}
\label{invoke}

In this section, a weight initialization method is introduced to make a shallow network (with depth around 50) address the difficulty of training.
This method will satisfy the \textit{norm-preserving} condition, $(\sigma_w^2\mu_1)=1$, mentioned in Section \ref{compare}.
With this kind of weight initialization, the network cannot even learn a simple task, such as the XOR problem.
Also, an analysis on the failed training is provided.

\subsection{Tweaking the VNI} \label{tweaking_vni}

In Section \ref{why}, we have already shown that when the VNI approaches 1, the network has a representation power that is inefficient to accomplish practical tasks.
Here, we provide a weight initialization method which satisfies the norm-preserving condition introduced in Section \ref{compare}, which can help us to tweak the initial VNI of a network to nearly 1:

\begin{algorithmic}[1]
\REQUIRE
  The dimension of the input vector, $N_i$\\
  The dimension of the output vector, $N_o$\\
  The dimension of the bottleneck vector, $N_b$
\ENSURE
  Initialized weight matrix, $\mathbf{W}$ with dimensions $(N_o, N_i)$\\
\STATE Randomly sample a matrix $\mathbf{U}$ with dimensions $(N_b, N_i)$ with zero mean and unit variance
\STATE Randomly sample a matrix $\mathbf{V}$ with dimensions $(N_o, N_b)$ with zero mean and unit variance
\STATE Evaluate $\widehat{\mathbf{W}}=\mathbf{V}\cdot \mathbf{U}$
\STATE Average number of inputs and outputs: $N_{mean}=(N_i+N_o)/2$
\STATE Normalization: $\mathbf{W}=\widehat{\mathbf{W}}/\sqrt{N_b\cdot N_{mean}}$
\end{algorithmic}

Note that the normalization obeys the norm-preserving constraint introduced by \cite{xavier}.

The bottleneck dimension $N_b$ is used to make the network much ``leaner''.
That is, when a small bottleneck dimension $N_b$ is chosen, the initial effective number of nodes will be smaller without modifying the network architecture, and hence we can have the VNI nearly 1 with a shallower network architecture.
If the bottleneck dimension $N_b$ is set to $1$, then due to Step 3 of the weight initialization method, the rank of the weight matrix $\mathbf{W}$ will become only 1.
That is, the weight matrix $\mathbf{W}$ will have only 1 non-zero eigenvalue.
By \eqref{rsq_eigen}, if all the weight matrices of a network are initialized by this method, then the VNI $R_{sq}$ will tend to 1.
Therefore, we can tune the network to make the VNI tend to 1 even when the depth is not large.

An architectural view of the weight initialization method is provided in Fig. \ref{fig:tweak_init}.
It is shown that if the weight matrices are initialized by the introduced method, the effective number of nodes will become only $N_b$.
If $N_b$ is set to 1, then the initial VNI $R_{sq}$ of the network will increase to 1.

\begin{figure}[ht]
\centering
\subfloat[Ordinary weight init.]{
  \centering
  \includegraphics[width=0.46\linewidth]{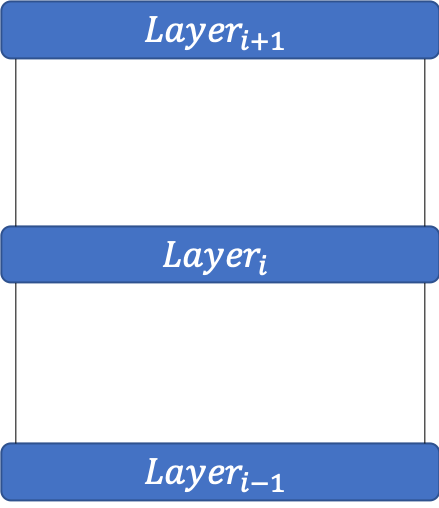}
  \label{fig:tweak_init_a}
}%
\subfloat[The introduced weight init.]{
  \centering
  \includegraphics[width=0.46\linewidth]{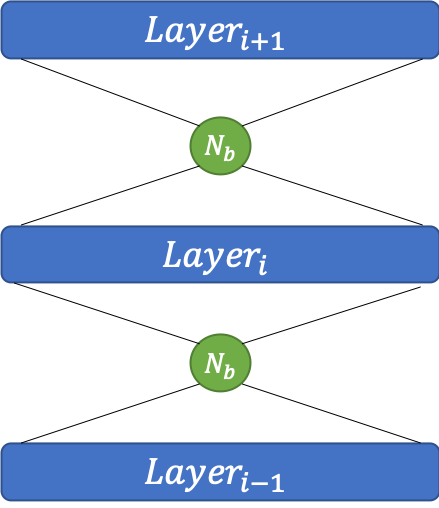}
  \label{fig:tweak_init_b}
}

\caption{
An architectural view of the introduced weight initialization.
Compared with ordinary weight initialization methods, the introduced weight initialization makes the initial weight matrix have a smaller rank $N_b$.
If the parameter $N_b$ is set to 1, then the effective number of hidden nodes will reduce to 1, which implies that the VNI $R_{sq}$ of the network will increase to 1.
}
\label{fig:tweak_init}
\end{figure}

The simulation results are provided as numerical evidences for this.
The networks are examined by the AND and the XOR tasks, where we use 2 or 4 bits as input data and the AND/XOR results
of the input data as the output label (with 2 or 4 classes), which can be written as $y_{and}^2 = x_0 \;\mathrm{AND}\; x_1$, $y_{and}^4 = (x_0 \;\mathrm{AND}\; x_1)\cdot 2+(x_2 \;\mathrm{AND}\; x_3)$ and $y_{xor} = x_0 \;\mathrm{XOR}\; x_1$.
A successful training is defined as when the testing accuracy
exceeds 0.99 (for MNIST, it's defined as 0.9) in 100 epochs over 20 runs.
The AND task can be viewed as a linearly-separable classification task, while the XOR task is a non-linearly-separable problem.

In the simulation, the network depth $L$ is set to 50 with the tanh activation function, the network width $N$ is set to 500, the bottleneck dimension $N_b$ is set to 30, and the network is trained with a $0.01$ learning rate.
The initial VNI of a network is tweaked to 1 by initializing the weight matrices with rank$=N_b$.
It is shown in Table \ref{tweak} that compared with networks with initial VNI $\neq1$, the networks to initial VNI $=1$
have insufficient representation power to learn either the XOR or the MNIST tasks.
Also, the metric $\sigma_w^2\mu_1$ is nearly 1 throughout the entire training, which implies that
the failures are not caused by the vanishing/exploding gradients problem.
That is, the vanishing nodes are indeed a problem when the VNI is nearly 1.

\begin{table}[ht]
    \caption{Tweaking the initial VNI to 1.
    }
    \label{tweak}
    \centering
    \begin{tabular}{|c|c|c|c|}
    \hline
        Tasks & \# of classes & Gaussian Init. & Initial VNI $=1$\\\hline
        AND & 2 & Success(VNI$\neq$1) & Success(VNI$=$1) \\\hline
        AND & 4 & Success(VNI$\neq$1) & Fail(VNI$=$1) \\\hline
        XOR & 2 & Success(VNI$\neq$1) & Fail(VNI$=$1) \\\hline
        MNIST & 10 & Success(VNI$\neq$1) & Fail(VNI$=$1) \\\hline
    \end{tabular}
\end{table}

\subsection{Walking dead gradients when the VNI is 1}

To analyze the failure of the training in the previous subsection, we first define a kind of gradients, to be referred to as \textit{walking dead} gradients.
The phrase ``walking dead'' consists of two words:
``walking'' means the gradients, which are used to update the weight matrices of the network, still have a non-vanishing value, but ``dead'' implies that the gradients cannot successfully train the network.
We observe that if the weight matrices are initialized as in the previous subsection, the initial VNI will be close to 1.

Below, the details of the training are provided in Table \ref{tweak}.
Table \ref{tweak} shows that although the network architectures (see Fig. \ref{fig:tweak_init}), the training environments, and the scales of the back-propagated gradients are the same,
the Gaussian initialized weight matrices can learn all the tasks
while the network with initial VNI = 1 fails and cannot escape from the occurrence of vanishing nodes.

\textbf{Walking}: As discussed in Section \ref{compare}, we have shown that the values of the gradients will not vanish if the initial weight condition ($\sigma_w^2\mu_1=1$) is met.
In the introduced weight initialization, it is obvious that the variance of the obtained weight matrix will be $Var[\mathbf{W}]=\frac{N_b\cdot Var[\mathbf{u}]\cdot Var[\mathbf{v}]}{N_b\cdot N_{mean}}$, which implies that the parameter $\sigma_w^2$ will become $Var[\mathbf{W}]\cdot N_{mean}=1$.

In Fig. \ref{fig:walking_ratio}, we would like to show that the scales of the back-propagated gradients of the scaled-Gaussian initialized network are similar to those of the network with initial VNI = 1. 
As in Table \ref{tweak}, the network depth $L$ is set to 50 with tanh activation function, the network width $N$ is set to 500, the bottleneck dimension $N_b$ is set to 30, and the network is trained with $0.01$ learning rate.
We record the gradient values at the input layer nodes over 20 runs, and then evaluate the log-scales of the gradients.
The log-ratio of the gradients between two different initial weights are provided in Fig. \ref{fig:walking_ratio}.
It can be seen that the gradients of both the scaled-Gaussian initialized network and the network with initial VNI = 1 have similar scales, which implies that the vanishing/exploding problems do not occur.
That is, the gradients still provide information from the back-propagation, but the gradients somehow cannot modify the network toward a successful model.

\begin{figure}
    \centering
    \includegraphics[width=1.0\linewidth]{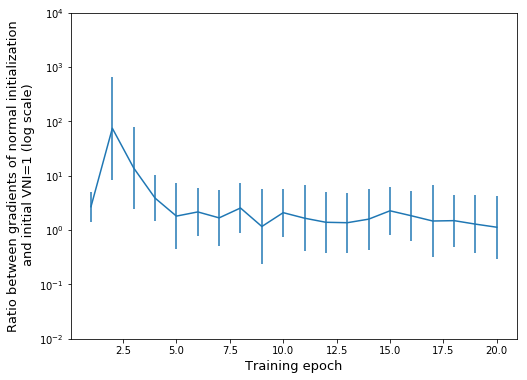}
    \caption{The log ratio of the back-propagated gradients between the scaled-Gaussian initialized network and the network with initial VNI = 1.}
    \label{fig:walking_ratio}
\end{figure}

\textbf{Dead}: In Table \ref{tweak}, the network architectures, optimization methods and the training environments for two initialization methods are totally the same.
However for the XOR task, the scaled-Gaussian initialized network can reach 100\% testing accuracy in around 5 epochs, while the testing accuracy of the network with initial VNI = 1 cannot exceed 90\% for over 100 epochs.
That is, the representation power of the network with initial VNI = 1 is not sufficient to learn the task, and therefore the network is called a \textit{dead} network.

\section{Experiments} \label{experiments}

\begin{figure*}
\centering
\newcommand{\myWidth}{0.18\textwidth}
\newcommand{\myspace}{\hspace{3mm}}
\subfloat[Tanh, Scaled Gaussian Init.]{
  \centering
  \includegraphics[width=\myWidth,trim={0 0 0 0.65cm},clip]{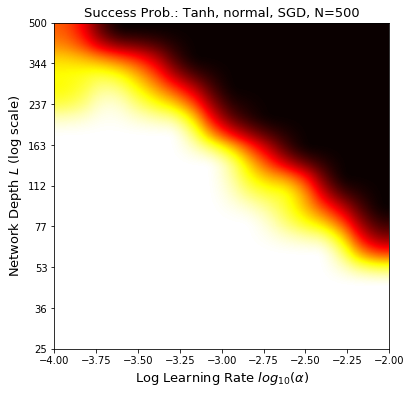}
  \label{fig:mnist_sim_s1}
}
\subfloat[ReLU, Scaled Gaussian Init.]{
  \centering
  \includegraphics[width=\myWidth,trim={0 0 0 0.65cm},clip]{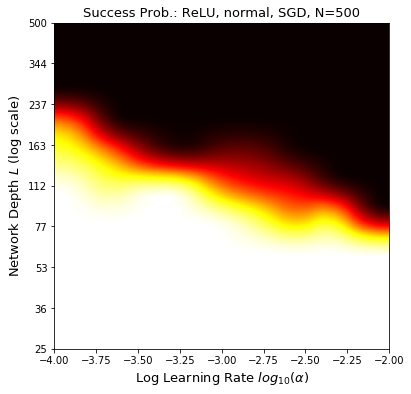}
  \label{fig:mnist_sim_s2}
}
\subfloat[Tanh, Orthogonal Init.]{
  \centering
  \includegraphics[width=\myWidth,trim={0 0 0 0.65cm},clip]{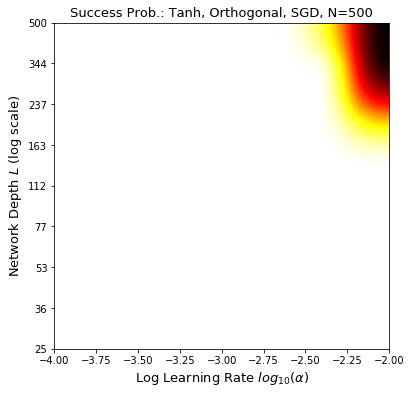}
  \label{fig:mnist_sim_s3}
}
\subfloat[ReLU, Orthogonal Init.]{
  \centering
  \includegraphics[width=\myWidth,trim={0 0 0 0.65cm},clip]{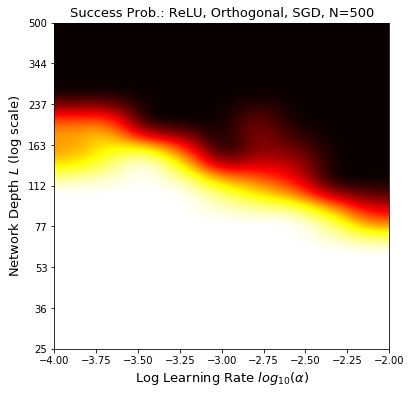}
  \label{fig:mnist_sim_s4}
}
\subfloat[Tanh, Householder Param.]{
  \centering
  \includegraphics[width=\myWidth,trim={0 0 0 0.0cm},clip]{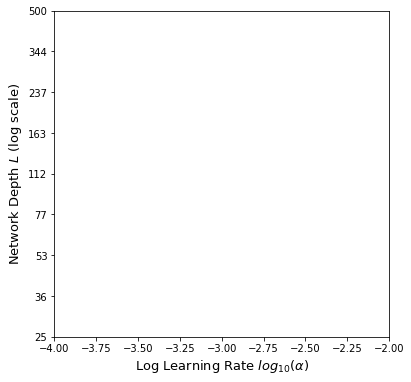}
  \label{fig:mnist_sim_s5}
}
\subfloat{
  \includegraphics[width=7mm]{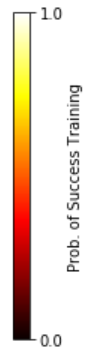}
}%
\caption{Probability of successful training for different network depths $L$ and learning rates $\alpha$ (SGD optimizer). Black denotes a zero probability of successful training.
}
\label{fig:mnist_sim}
\end{figure*}

To empirically explore the effects of the phenomenon of vanishing nodes on the training of deep neural networks, we have performed experiments with the training tasks on the MNIST dataset \cite{mnist}. Because the purpose is to focus on the problem of vanishing nodes, the networks are designed so that vanishing/exploding gradients will never occur; that is, they are initialized with weights ($\sigma_w^2\mu_1=1$).
The network is trained with a batch size of 100.
The number of successful trainings for a total of 20 runs is recorded to reflect the influence of vanishing nodes on the training process, which may lead to insufficient network representation capability, as shown in Fig. \ref{fig:sec5_sim1}.
A successful training is considered to occur when the training accuracy exceeds 90\% within 100 epochs. 
The network depth $L$ ranges from $25$ to $500$, and the network width $N$ is set to $500$.
The learning rate $\alpha$ ranges from $10^{-4}$ to $10^{-2}$ with the SGD algorithm.
Both $L$ and $\alpha$ are uniformly distributed on the logarithmic scale.
The experiments were performed in the MXNet framework\cite{mxnet}.

Considering that the MNIST dataset is too simple a task,
we would say that we intentionally chose an easy training task in order to claim that when the VNI
increases to 1, the deep network cannot even learn a simple task.
If vanishing nodes make the network fail at such a simple task, then a more challenging task should be even less possible to carry out.


Fig. \ref{fig:mnist_sim} shows the results of two different activation functions (Tanh/ReLU) with two different weight initializations (scaled-Gaussian/orthogonal from \cite{mft:linear}). When a network with tanh activation functions is initialized with orthogonal weights, the term of  $(\mu_2/\mu_1^2-1-s_1)$ in \eqref{rsq_moment} becomes zero. Therefore, its $R_{sq}$ will be the minimum value ($1/N$) and will not depend on the network depth. For the other network parameters, $(\mu_2/\mu_1^2-1-s_1) $ will not equal zero, and $R_{sq}$ still depends on the network depth. The experimental results show that the likelihood of a failed training is high when the depth $L$ and the learning rate are large. In addition, the corresponding $R_{sq}$ of failed cases becomes nearly $1$, which causes a lack of network representation power.
This implies that the vanishing nodes problem is the main reason that the training fails. A comparison of Fig. \ref{fig:mnist_sim_s3} with the other three results shows clearly that the networks with the minimum $R_{sq}$ value have the highest successful training probability.

Shallow network architectures can tolerate a greater learning rate, which is why the vanishing node problem has been ignored in many networks with small depths. In a deep network, the learning rate should be set to small values to prevent $R_{sq}$ from increasing to 1.
The reason why the behavior of $R_{sq}$ is affected by the learning rate $\alpha$ remain unexplained, suggesting the need for further investigations to better understand the relation between the learning rate and the dynamics of $R_{sq}$.
A high learning rate will cause $R_{sq}$ to be severely intensified to nearly 1, and the representation capability of the network will be reduced to that of a single perceptron, which is the main reason that the training fails.
Further analyses of the cause of failed training are provided in the following subsection.

\subsection{The cause of failed training}

In this subsection, we analyze the reason why failed training occurs from two perspectives: vanishing/exploding gradients and vanishing nodes.
First, the quantity $\sigma_w^2$ (the variance of weights at each layer ) of trained models is evaluated.
There are a total of 31,680 runs in the experiments, including 13,101 failures and 18,579 successes.
Detailed information is presented in Table \ref{num_table}.
The quantity $\sigma_w^2\mu_1$ for measuring the degree of vanishing/exploding gradients is presented in Fig. \ref{fig:succ_box} and Fig. \ref{fig:fail_box} for successful networks and failed networks.
The two figures display box and whisker plots to represent the distribution of $\sigma_w^2\mu_1$ at each network layer, and the horizontal axis indicates the depth of the trained network. 
The results show that both successful and failed networks have $\sigma_w^2\mu_1$ near one.
This indicates that both successful and failed models satisfy a condition of that prevents the networks from having vanishing/exploding gradients.

Second, the differences in $R_{sq}$ between successful and failed networks are displayed in Fig. \ref{fig:histogram}.
The horizontal axis indicates the $R_{sq}$ of trained models evaluated by \eqref{rsq_def}, and the vertical axis represents the binned frequencies of $R_{sq}$.
The blue histogram represents the $R_{sq}$ of the failed networks, and the orange histogram represents the $R_{sq}$ of the successful networks.
The $R_{sq}$ of the failed models ranges from 0.9029 to 1.0000 with a mean of 0.9949 and a standard deviation of 0.0481, and that of the successful models ranges from 0.1224 to 0.9865 with a mean of 0.3207 and a standard deviation of 0.1690.
The figure shows that the $R_{sq}$ of the failed networks mainly cluster around 1, but those of the successful networks are widely distributed.
From the analysis shown in Figs. \ref{fig:succ_box}, \ref{fig:fail_box} and \ref{fig:histogram}, it is clear that the vanishing nodes ($R_{sq}$ reaches 1) is the main cause which makes the training fail.

\begin{figure}[ht]
    \centering
    \includegraphics[width=1.0\linewidth]{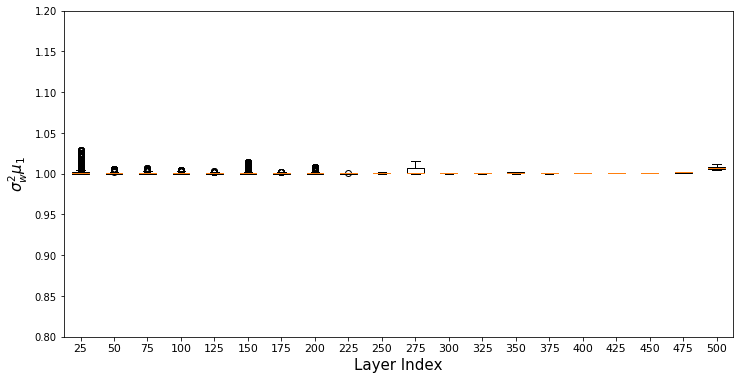}
    \caption{Box and whisker plot of $\sigma_w^2\mu_1$ for networks with successful training. There are 18,579 successful runs.
    }
    \label{fig:succ_box}
\end{figure}

\begin{figure}[ht]
    \centering
    \includegraphics[width=1.0\linewidth]{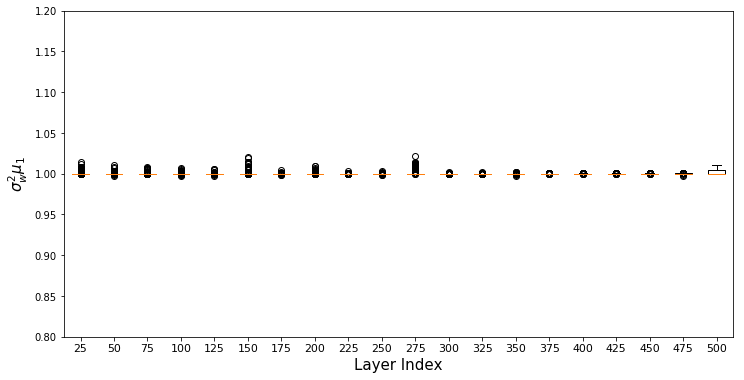}
    \caption{Box and whisker plot of $\sigma_w^2\mu_1$ for networks with failed training. There are 13,101 failed runs.
    }
    \label{fig:fail_box}
\end{figure}

\begin{figure}[ht]
    \centering
    \includegraphics[width=1.0\linewidth]{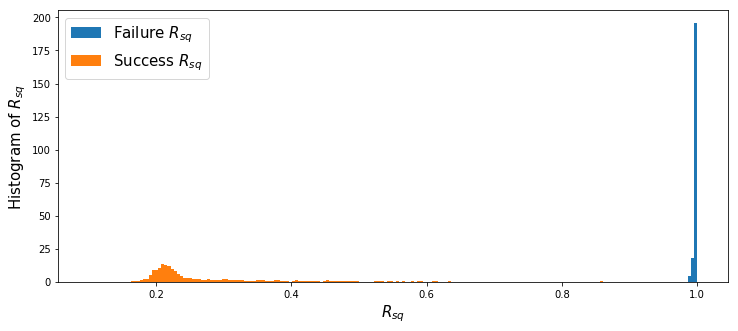}
    \caption{Histogram of $R_{sq}$ of failed/successful  networks.}
    \label{fig:histogram}
\end{figure}

\begin{table*}[ht]
    \centering
    \begin{tabular}{|c|c|c|r|r|}
    \hline
         Optimizer&Activation&Weight Init.&No. of Failure&No. of Success\\\hline
        
         \multirow{4}{*}{SGD}&Tanh&
         \multirow{2}{*}{Scaled Gaussian}&712&1268\\
         \cline{2-2}\cline{4-5}
         &ReLU&&1775&205\\\cline{2-5}
         &Tanh&\multirow{2}{*}{Orthogonal}&62&1918\\\cline{2-2}\cline{4-5}
         &ReLU&&882&1098\\\hline

         \multirow{4}{*}{SGD+momentum}&Tanh&\multirow{2}{*}{Scaled Gaussian}&982&998\\\cline{2-2}\cline{4-5}
         &ReLU&&1140&840\\\cline{2-5}
         &Tanh&\multirow{2}{*}{Orthogonal}&546&1434\\\cline{2-2}\cline{4-5}
         &ReLU&&1044&936\\\hline

         \multirow{4}{*}{Adam}&Tanh&\multirow{2}{*}{Scaled Gaussian}&609&1371\\\cline{2-2}\cline{4-5}
         &ReLU&&723&1257\\\cline{2-5}
         &Tanh&\multirow{2}{*}{Orthogonal}&354&1626\\\cline{2-2}\cline{4-5}
         &ReLU&&1465&515\\\hline
      
         \multirow{4}{*}{RMSProp}&Tanh&\multirow{2}{*}{Scaled Gaussian}&763&1217\\\cline{2-2}\cline{4-5}
         &ReLU&&827&1153\\\cline{2-5}
         &Tanh&\multirow{2}{*}{Orthogonal}&453&1527\\\cline{2-2}\cline{4-5}
         &ReLU&&764&1216\\\hline
    \end{tabular}
    \caption{The detailed numbers of successful and failed runs.}
    \label{num_table}
\end{table*}

In Fig. \ref{fig:mnist_sim_s3}, the dark region at the top right corner reveals that although the weight matrices are initially orthogonal, after the training they turn out to be non-orthogonal with a VNI close to 1.
As a motivation, we restrict the weights to orthogonal matrices by the Householder parameterization\cite{householder}.

\subsection{Orthogonal constraint: The Householder parametrization}
An $n$-dimensional weight matrix following the Householder parametrization can be represented by the product of $n$ orthogonal Householder matrices\cite{householder}.
That is, $\mathbf{W}=\mathbf{H}_n\cdot\mathbf{H}_{n-1}\cdots\mathbf{H}_1$, where the $\mathbf{H}_i$ are defined by $\mathbf{I}-2\mathbf{v}_i\mathbf{v}_i^T$ and the $\mathbf{v}_i$ are vectors with unit norm.

The results of the Householder parameterization are provided in Fig. \ref{fig:mnist_sim_s5}, where the disappearance of the dark region implies that a successful training can be attained if the weight matrices remain orthogonal during the training process.
Also, we provide other results in Table \ref{orth}.
It shows that for the MNIST classification, there is training success for 20 runs even with a large learning rate ($10^{-2}$) and a large network depth ($500$) are used, which the orthogonal initialization cannot attain.
This shows that forcing the VNI be small during the whole of the training process can prevent the network from failing at training, and it also implies that the failure cases are indeed caused by the vanishing nodes.

\begin{table*}[ht]
    \caption{Forcing the weights to remain orthogonal throughout the entire training process.
    }
    \label{orth}
    \centering
    \begin{tabular}{|c|c|c|c|}
    \hline
        Number of layers $L$ & Scaled Gaussian Init. & Scaled Orthogonal Init. &
        Householder Parametrization\\\hline
        50 & Success (VNI $\neq1$) & Success (VNI $\neq1$) & Success (VNI $\neq1$) \\\hline
        200 & Fail (VNI $=1$) & Success (VNI $\neq1$) & Success (VNI $\neq1$) \\\hline
        500 & Fail (VNI $=1$) & Fail (VNI $=1$) & Success (VNI $\neq1$) \\\hline
    \end{tabular}
\end{table*}

\section{Conclusion} \label{conclusion}

The phenomenon of \textit{vanishing nodes} has been investigated as another challenge  when training deep networks.
Like the vanishing/exploding gradients problem, vanishing nodes also make training deep networks difficult.
The hidden nodes in a deep neural network become more correlated as the depth of the network increases, as the similarity between the hidden nodes increases.
Because a similarity between nodes results in redundancy, the effective number of hidden nodes in a network decreases.
This phenomenon is called\textit{vanishing nodes}.

To measure the degree of vanishing nodes, the \textit{vanishing nodes indicator} (VNI) is proposed.
It is shown theoretically that the VNI is proportional to the depth of the network and inversely proportional to the width of the network, which is consistent with the experimental results.
Also, it is mathematically proven that when the VNI equals 1, the effective number of nodes of the network vanishes to 1, and hence the network cannot learn some simple tasks.

Moreover, we explore the difference between vanishing/exploding gradients and vanishing nodes, and another weight initialization method which prevents the network from having vanishing/exploding gradients is proposed to tweak the initial VNI to 1.
The numerical results of the introduced weight initialization reveal that when the VNI is set to 1, a relatively shallow network may still not to be able to not accomplish an easy task.
The back-propagated gradients of the network still have sufficient magnitudes, but cannot make the training succeed, and therefore we say that the gradients are \textit{walking dead}.

Finally, experimental results show that the training fails when there are vanishing nodes, but that if an orthogonal weight parametrization is applied to the network, the problem of vanishing nodes will be eased, and the training of the deep neural network will succeed.
This implies that the vanishing nodes are indeed the cause of the difficulty of the training. 

\subsection*{Acknowledgement}

This research is partially supported by the "Cyber Security Technology Center" of National Taiwan University (NTU), sponsored by the Ministry of Science and Technology, Taiwan, R.O.C. under Grant no. MOST 108-2634-F-002-009.

\bibliographystyle{plain}
\bibliography{references}

\ifieee
\begin{IEEEbiographynophoto}{Wen-Yu Chang}
Biography text here.
\end{IEEEbiographynophoto}

\begin{IEEEbiographynophoto}{Tsung-Nan Lin}
Biography text here.
\end{IEEEbiographynophoto}
\fi

\end{document}